%% file: main.tex
\documentclass[letterpaper, 10 pt, conference]{ieeeconf}  

\IEEEoverridecommandlockouts                              
\title{\LARGE \bf
Robustifying Binary Classification to Adversarial Perturbation
}

\author{Fariborz Salehi and Babak Hassibi
\thanks{F. Salehi and B. Hassibi are with the Department of Electrical Engineering, California Institute of Technology, Pasadena, CA 91125, USA. \{fsalehi, hassibi\}@caltech.edu}
}
\include{defs}

\begin{document}

\maketitle
\thispagestyle{empty}
\pagestyle{empty}

\begin{abstract}

Despite the enormous success of machine learning models in various applications, most of these models lack resilience to (even small) perturbations in their input data. Hence, new methods to robustify machine learning models seem very essential. To this end, in this paper we consider the problem of binary classification with adversarial perturbations. Investigating the solution to a min-max optimization (which considers the worst-case loss in the presence of adversarial perturbations) we introduce a generalization to the max-margin classifier which takes into account the power of the adversary in manipulating the data. We refer to this classifier as the "Robust Max-margin" (RM) classifier. Under some mild assumptions on the loss function, we theoretically show that the gradient descent iterates (with sufficiently small step size) converge to the RM classifier in its direction. Therefore, the RM classifier can be studied to compute various performance measures (e.g. generalization error) of binary classification with adversarial perturbations.

\end{abstract}

\section{Introduction}
\label{sec:intro}
Machine learning models have been very successful in many applications,  ranging from spam detection, speech and visual recognition, to the analysis of genome sequencing and financial markets. Yet, despite this indisputable success, it has been observed that commonly used machine learning models (such as deep neural networks) are very instable in the presence of non-random perturbations~\cite{szegedy2013intriguing, biggio2013evasion, carlini2017towards}.\\
The instability of machine learning models is a fundamental issue that needs to be addressed, especially when such models are used in sensitive applications such as autonomous systems. There have been many recent efforts to address this issue (a partial list of papers includes~\cite{madry2017towards, xu2017feature, shafahi2019adversarial}.) However, robustness comes at a cost and it is often the case that the adversarial training algorithms underperform on the clean data when compared with their (non-robust) counterparts. Understanding the tradeoffs (in accuracy) between the robust and standard models is an important problem the answer to which can help us find more efficient training methods. Recently, Javanmard et. al.~\cite{javanmard2020precise}  precisely characterized the tradeoff between standard and adversarial risks for the linear regression problem. They also analyze the performance of the resulting model under i.i.d. Gaussian training data.\\
In this paper, we study the simple (yet fundamental) problem of binary classification where the goal is to find a classifier that has a high accuracy in predicting the binary labels when having feature vectors as its input. When the clean data is available, max-margin classifier~\cite{vapnik1982estimation} is the model of choice as maximizing the margin is interpreted as minimizing the risk of misclassification~\cite{cortes1995support}. Recently, it was shownin ~\cite{soudry2018implicit} that for a broad class of loss functions, including the well-known logistic loss, the gradient descent iterates converge to the max-margin classifier. More recently, the asymptotic performance of this classifier has been characterized in~\cite{montanari2019generalization, deng2019model, salehi2020gmm}. \\
We consider the case where the training data is perturbed by an adversary and introduce the "Robust Max-margin" (RM) classifier as a generalization of max-margin to perturbed input data. We then consider the adversarial training method, in which the optimal parameter is a solution to a saddle-point optimization. We show that the gradient descent algorithm with properly-tuned step sizes converges in its direction to the RM classifier. A significant consequence of this result is that one can characterize various  performance measures (e.g. generalization error) of adversarial training in binary classification by analyzing the performance of the RM classifier.\\
To the extent of our knowledge, this is the first work that introduces the robust max-margin classifier and proves the convergence of gradient descent iterates to this classifier. This paper was originally submitted on March $2020$ to the Conference on Decision and Control (CDC). We should note that more recently in~\cite{javanmard2020binary}, the authors have shown similar results (referred to as the "robust separation") and analyze the performance of the resulting classifier under i.i.d. Gaussian training data. Their analysis on the performance of the resulting estimator is based on the Convex Gaussian Min-max Theorem~\cite{stojnic2013framework, thrampoulidis2015regularized}. Similar analyses have been recently provided for the performance of max-margin classifiers as well as other generalized linear models~\cite{salehi2019impact, montanari2019generalization, deng2019model, emami2020generalization, salehi2020gmm}.\\
The organization of the paper is as follows: In Section~\ref{sec:prelim} we provide some background on the binary classification problem and how it connects with the max-margin classifier. The mathematical setup for the problem of binary classification with perturbed training data is provided in Section~\ref{sec:setup}. The main result of the paper is presented in Section~\ref{sec:main},  and the proofs are provided in Sections~\ref{sec:proof_lemma} and~\ref{sec:proof_main}.

\section{Preliminaries}
\label{sec:prelim}
\subsection{Notations}
For any vector $\w\in \R^p$, the binary classifier associated with $\w$ is defined as: $C_{\w}:\mathbb R^p\rightarrow \{\pm 1\}$, such that $C_{\w}(\x) = \text{Sign}(\w^T\x)$. $\mathbb N$ denotes the set of non-negative integers. For a vector $\v$, $\v^T$ denotes its transpose, and $\lnorm{\v}_p$ (for $p\geq 1$) is its $\ell_p$ norm, where we often omit the subscript for $p=2$. $\sigma_{\max}(\mathbf M)$ denotes the maximum singular value of the matrix $\mathbf M$. $\mathbf 0_d$ and $\mathbf 1_d$ respectively represent the all-one and all-zero vectors in dimension $d$. A function $f(\cdot)$ is said to be $L$-smooth if its derivative, $f'(\cdot)$, is $L$-Lipschitz. 
\subsection{Background: binary classification with unperturbed data}
Let $\mathcal D=\{(\x_i, y_i): 1\leq i\leq n \}$ denote a set of data points, where for $i=1,\ldots,n$, $\x_i \in \mathbb R^p$ is the feature vector, and $y_i \in \{\pm1\}$ is the binary label. We assume that $\mathcal D$ is linearly separable, i.e., there exist $\w^\star \in \R^p$ such that:
\begin{equation}
    y_i = \text{Sign}(\x_i^T\w^\star)~,~\text{for } i=1,2,\ldots,n.
\end{equation}
When the training data has no perturbation, one can attempt to find a classifier by minimizing the empirical loss on dataset $\mathcal D$. In the setting of binary classification, the loss function is usually formed as,
\begin{equation}
    \mathcal L(\w) = \sum_{i=1}^{n} \ell(y_i\x_i^T\w)
\end{equation}
where the function $\ell(\cdot):\R\rightarrow \R_+$ is a decreasing function that approaches $0$ as its input approaches infinity. A typical approach to find the minimizer of the loss function $\mathcal L(\w)$ is through the iterative algorithms, such as the gradient descent (GD) algorithm. 
The convergence of the GD iterates on separable datasets has been studied in recent papers~\cite{ji2018risk, soudry2018implicit}, where it was shown, among others, that while the norm of the iterates approaches infinity, their direction would approach to the direction of the well-known $L_2$ max-margin classifier defined as,
\begin{equation}
    \label{eq:max-margin}
    \begin{aligned}
    && \w_{M} = \arg \min_{\w \in \R^p}&~~\lnorm{\w}\\
    &&& \text{s.t.}~~y_i\x_i^T\w \geq 1~,~~1\leq i\leq n.
    \end{aligned}
\end{equation}
In other words, their result states that for almost every $\x\in \R^p$, $C_{\w_t}(\x)\rightarrow C_{\w_{M}}(\x)$ as $t$ grows, where $\w_t$ denotes the result of GD after $t$ steps. The max-margin classifier~\eqref{eq:max-margin} (a.k.a. hard-margin SVM~\cite{cortes1995support}) has been extensively studied in the machine learning community. This classifier simply maximizes the smallest distance of the data points to the separating hyperplane (referred to as the margin).\\
The abovementioned result, i.e., convergence of the GD iterates to the max-margin classifier, has significant consequences as the max-margin classifier can then be studied to compute various performance measures (such as the generalization error) of the resulting estimator. Very recently, researchers have exploited this result to accurately compute the generalization error of GD over the logistic loss~\cite{montanari2019generalization}.  
\section{Binary classfication with adversarial perturbation}\label{sec:setup}
As explained earlier in Section~\ref{sec:intro}, understanding the behavior of machine learning models under perturbed input is very essential with the goal of improving the robustness of these models. Inspired by recent advances in understanding the behavior of machine learning models under adversarial perturbation, here we study the problem of binary classification with perturbed data.\\
We assume that the training data is a perturbed version of the underlying dataset, $\mathcal D$.
Let $\mathcal D'=\{(\x_i+\mathbf z_i, y_i): 1\leq i\leq n\}$ denote the set of training data, where, for $i=1,2,\ldots,n$, $\mathbf z_i\in \mathcal S_i$ is the unknown perturbation, and the set $\mathcal S_i$ consists of all the allowed perturbation vector. In the adversarial setting it is often assumed that the perturbation vectors, $\{\mathbf z_i\}_{i=1}^{n}$, are chosen in such a way that the training algorithm is beguiled into generating a wrong solution. \\
Throughout this paper, we assume that the perturbation vectors have bounded norms by defining $\mathcal S_i = \epsilon_i \mathcal B_p$, where $\mathcal B_p$ denotes the unit ball in $\mathbb R^p$, and $\epsilon_i\geq 0$, for $1\leq i\leq n$, indicates the maximum allowed norm for the $i$-th perturbation vector, $\mathbf z_i$. While the perturbation vectors are hidden to us, we assume having knowledge of $\{\epsilon_i\}_{i=1}^{n}$. \\
Note that the set of allowed perturbations can be different for different data points. This includes certain special cases such as: (1) only a subset of the data is perturbed ($\epsilon_i = 0$ if the $i$-th data point is not perturbed), and (2) all the data points have the same perturbation set, i.e., for some $\epsilon\geq 0$ we have $\epsilon_i = \epsilon$ for $1\leq i\leq n$, .
\subsection{Saddle-point optimization}
\label{sec:saddle_point_opt}
The parameters of the desired model are often derived by forming a loss function and solving an optimization problem to find a minimizer of the loss. In adversarial training, one should also consider the manipulative power of the adversary where the adversary attempts to misguide the training algorithm. When the goal of a training algorithm is to minimize a loss function, one can view the adversary as an entity which attempts to maximize the loss. The following  $\min$-$\max$ optimization problem incorporates the contrary behaviors of the adversary and the training algorithm with respect to the loss function.
\begin{equation}
    \label{eq:min-max optimization}
    \min_{\w \in \mathbb R^p}~\max_{\mathbf z_i \in \mathcal S_i, 1\leq i\leq n}~\mathcal L(\w):= \sum_{i=1}^{n}\ell\big(y_i(\x_i+\mathbf z_i)^T\w\big).
\end{equation}
In order to find a robust model, we should solve this saddle-point optimization. Under our assumptions on the perturbation sets, we can introduce the function $\mathcal L_{\bm \epsilon}(\w)$ which is the result of the inner maximization in~\eqref{eq:min-max optimization}, i.e.,
\begin{equation}
\label{eq:robust_loss}
    \mathcal L_{\bm \epsilon}(\w) = \sum_{i=1}^{n} \max_{\lnorm{\mathbf z_i}\leq \epsilon_i}\ell\big(y_i(\x_i+\mathbf z_i)^T\w\big),
\end{equation}
where $\bm \epsilon = [\epsilon_1,\epsilon_2,\ldots,\epsilon_n]^T$. Therefore, the robust classifier is defined as a minimizer of $\mathcal L_{\bm \epsilon}(\w)$.

\section{Main Results}
\label{sec:main}
In this section, we present the main results of the paper. First, in Section~\ref{sec:Robust_max_margin} we introduce the {\bf{R}}obust {\bf{M}}ax-margin ({\bf{RM}}) classifier as an extension of the max-margin classifier when the training data is perturbed. Consequently, in Section~\ref{sec:GD_Iterates}, we show that, under some conditions on the function $\ell(\cdot)$, gradient descent algorithm (with sufficiently small step size) would converge in its direction to the RM classifier.  
\subsection{Robust Max-margin (RM) Classifier}
\label{sec:Robust_max_margin}
The max-margin classifier is a classifier that maximizes the minimum distance of the data points to the separating hyperplane (margin). In our setting where the training data is perturbed we should modify the notion of the margin to incorporate various perturbations across data points. More specifically, in order to get a robust classifier we would like the data points with higher perturbations to be farther away from the resulting separating hyperplane.

The {\bf{Robust Max-margin}} classifier is defined as, 
\begin{equation}
    \label{eq:robust_max-margin}
    \begin{aligned}
    &&& \w_{RM}^{(\bm\epsilon)} := \arg \min_{\w \in \R^p}~~\lnorm{\w}\\
    &&&~~~~~~~~~~\quad \text{s.t.}~~y_i\x_i^T\w \geq 1+\epsilon_i\lnorm{\w}~,~~1\leq i\leq n.
    \end{aligned}
\end{equation}
\begin{figure}[t]
      \centering
      \includegraphics[scale=0.22]{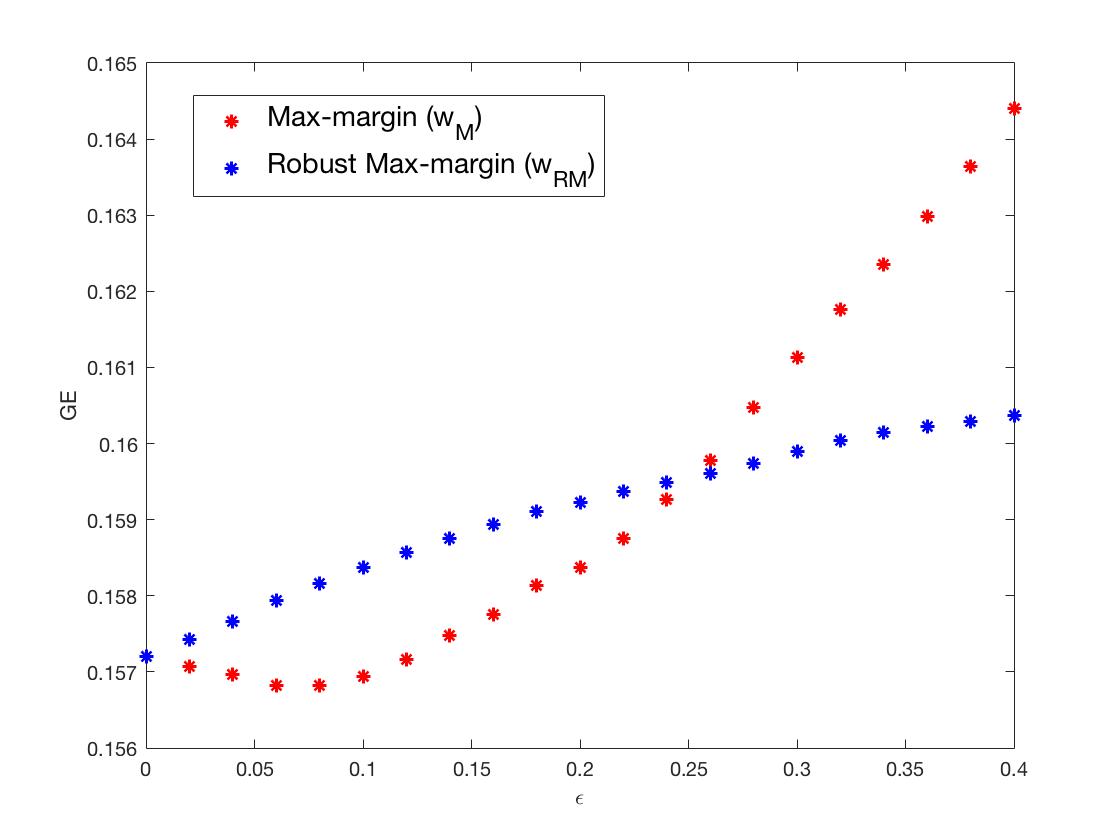}
      \caption{A comparison in generalization error (GE) between the max-margin~\eqref{eq:max-margin} and the robust max-margin~\eqref{eq:robust_max-margin}. The result is the average over $20$ independent trials with $n=100$ and $p=40$. The data is generated from a Gaussian distribution and $40\%$ of data points are perturbed with maximum norm of $\epsilon$. For large values of $\epsilon$, the RM classifier has a better generalization error than the max-margin classifier.}
      \label{fig:1}
\end{figure}
As observed in the constraints of this optimization, the RM classifier enforces data points with higher perturbations to keep a larger distance from the separating hyperplane $\{\x:\w_M^T\x = 0\}$. \\
When the data is perturbed, we expect the RM classifier to outperform the max-margin classifier. Figure~\ref{fig:1} depicts a comparison in generalization error between the max-margin and the RM classifier. Although for small perturbations, the two model behave the same, the RM classifier has a better performance as we increase the norm of perturbations.\\
While the separability of the data is necessary for the existence of the RM classifier, it is not sufficient. The following lemma provides a sufficient condition for its existence.
\begin{lem}
The RM classifier exists when the data set $\mathcal D=\{(\x_i,y_i):1\leq i \leq n\}$ is separable and,
\begin{equation}
    \lnorm{\bm\epsilon}_{\infty}<\frac{1}{\lnorm{\w_M}_2}~,
\end{equation}
where $\w_M$ is the max-margin classifier.
\end{lem}
\begin{proof}
The max-margin classifer, $\w_M$, exists when $\mathcal D$ is linearly separable. Also, $\bar \w = \frac{1}{1-\lnorm{\bm\epsilon}_{\infty}\lnorm{\w_M}_2}\w_M $ is a feasible point of the optimization~\eqref{eq:robust_max-margin}. Therefore, the RM classifier exists and $\lnorm{\w_M}\leq \lnorm{\w_{RM}}\leq \lnorm{\bar \w}$.
\end{proof}
When the perturbation sets are the same for different data points, one expects the RM classifier to be the same as the max-margin classifier. 
\begin{lem}
If $\bm \epsilon=\epsilon \times \mathbf 1_{n}$ for some $\epsilon\geq0$, the RM classifier exists if and only if $\epsilon<\frac{1}{\lnorm{\w_M}}$. In this case, 
\begin{equation}
    \w_{RM} = \frac{\w_M}{1-\epsilon \lnorm{\w_M}}.
\end{equation}
\end{lem}
\begin{proof}
Assume $\w_{RM}$ exists, then we have $\bar\w = \frac{\w_{RM}}{1+\epsilon\lnorm{\w_{RM}}}$ satisfies the constraints in the optimization~\eqref{eq:max-margin}. Since $\w_M$ is the solution to this optimization, we have $\lnorm{\w_M}\leq \lnorm{\bar\w}$ which gives $\epsilon\cdot\lnorm{\w_M}<1$. It is easy to check that $\w = \frac{\w_M}{1-\epsilon \lnorm{\w_M}}$ is the solution to the optimization~\eqref{eq:robust_max-margin}, as it satisfies the constraints and $\w_M$ is the optimal value of the optimization program~\eqref{eq:max-margin}.
\end{proof}
\subsection{Convergence of GD Iterates}
\label{sec:GD_Iterates}
In this section, we present the main result of the paper that is the convergence of the gradient descent iterates to the RM classifier. As discussed earlier in Section~\ref{sec:saddle_point_opt}, the goal is to solve the following optimization problem.
\begin{equation}
    \min_{\w \in \R^p} \mathcal L_{\bm \epsilon}(\w),
\end{equation}
where $\mathcal L_{\bm \epsilon}(\cdot)$ is defined in~\eqref{eq:robust_loss}. Gradient descent (GD) is the common method of choice to find a minimizer of this optimization. Starting from an initialization, $\w_0\in \R^p$, the GD iterates are generated through the following update rule:
\begin{equation}
    \label{eq:GD_iterates}
    \w_{t+1} = \w_{t} - \eta\cdot \nabla\mathcal L_{\bm\epsilon} (\w_t),~\text{for}~t\in \mathbb N,
\end{equation}
where $\eta>0$ is the step size.

Our goal is to study the behavior of the GD iterates as $t$ grows large. For our analysis, we need some assumptions to hold for the loss function $\ell(\cdot)$.
\begin{assump}
\label{ass_1}
The function $\ell:\mathbb R\rightarrow \mathbb R_+$ is twice-differentiable, monotonically decreasing, and $\beta$-smooth.
\end{assump}
We note that the common choices of the loss function satisfy the conditions in Assumption~\ref{ass_1}. For instance, the logistic loss defined as $\ell(u) = \log\big(1+\exp(-u)\big)$ satisfies these conditions (with $\beta=1$.)
We first state the following lemma which provides some insights on the behavior of GD iterates, $\w_t$, as $t\rightarrow \infty$. 
\begin{lem}
\label{lem:lim_w}
Consider the gradient descent iterates~\eqref{eq:GD_iterates} with step size $\eta<2\cdot \beta^{-1}\cdot(\sigma_{\max}(\mathbf X)+\lnorm{\bm \epsilon})^{-2}$, where $\mathbf X=[\x_1,\x_2,\ldots,\x_n]^T\in \R^{n\times p}$ is the data matrix, $\mathcal L_{\bm \epsilon}$ is defined in~\eqref{eq:robust_loss}, and $\ell(\cdot)$ satisfies Assumption~\ref{ass_1}. If the RM classifier exists, then, as $t\rightarrow +\infty$ we have,
\begin{enumerate}[i.]
    \item $~\lnorm{\w_t} \rightarrow +\infty$,
    \item $~\nabla\mathcal L_{\bm \epsilon}(\w_t)\rightarrow \mathbf 0_p$ , and,
    \item  $~y_i\x_i^T\w_t - \epsilon_i \lnorm{\w_t} \rightarrow +\infty$, for $i=1,2,\ldots,n$.
\end{enumerate}

\end{lem}

The proof of this lemma is provided in Section~\ref{sec:proof_lemma}. Lemma~\ref{lem:lim_w} provides useful insights on the behavior of the gradient descent iterates. With small enough step size, as $t$ grows the norm of $\w_t$ becomes unbounded while making $\mathcal L(\w_t)$ closer to zero. Since $\w_t$ diverges, we focus our attention on its direction, i.e., the normalized vector $\frac{\w_t}{\lnorm{\w_t}}$. In fact, the classifier defined by $\w_t$, $C_{\w_t}(\cdot)$, only depends on its direction. Therefore, if $\frac{\w_t}{\lnorm{\w_t}}$ converges, we can claim that the classifiers generated by GD iterates converge. \\
Our main result in Theorem~\ref{thm:main} states that the classifiers generated form the GD iterates converges to the RM classifier defined in Section~\ref{sec:Robust_max_margin}. Before stating this result, we need the following definition which is a modified version of an assumption in~\cite{soudry2018implicit}.
\begin{defn}
A function $f(u)$ has a tight exponential tail if there exist positive constants $a,c, \tau,\mu$ such that for all $u>\tau$:
\begin{equation}
\label{eq:def_exp_tight}
    \begin{cases}
    f(u)\leq c\big(1+\exp(-\mu\cdot u)\big)\exp(-a\cdot u),~\text{and},\\
    f(u)\geq c\big(1-\exp(-\mu\cdot u)\big)\exp(-a\cdot u).
    \end{cases}
\end{equation}
\end{defn}
\begin{theorem}
\label{thm:main}
Let Assumption~\ref{ass_1} holds and $-\ell'(\cdot)$ has a exponential tail. Consider the gradient descent iterates in~\eqref{eq:GD_iterates} with $\eta<2\cdot \beta^{-1}\cdot(\sigma_{\max}(\mathbf X)+\lnorm{\bm \epsilon})^{-2}$. Then, for almost every dataset we have,
\begin{equation}
    \lim_{t\rightarrow \infty}\lnorm{\frac{\w_t}{\lnorm{\w_t}} - \frac{\w_{RM}}{\lnorm{\w_{RM}}}} = 0.
\end{equation}
Threfore, the resulting classifier converges to the RM classifier.
\end{theorem}
\begin{rem}
The assumption on $-\ell'(\cdot)$ having a tight exponential tail holds for common loss functions in binary classification. As an example, the derivative of the logistic function satisfies~\eqref{eq:def_exp_tight} with $a=c=\mu = 1$. 
\end{rem}
\begin{rem}
Theorem~\ref{thm:main} states that while $\w_t$ diverges as $t$ grows , its direction converges to the direction of the robust max-margin classifier. We should note that this convergence is quite slow. Figure~\ref{fig:2} depicts the convergence of the direction of GD iterates to the RM classifier as $t\rightarrow \infty$ where it can be observed the convergence becomes slow as $t$ grows (the horizontal axis has a logarithmic scale.) In our proof in Section~\ref{sec:proof_main} we theoretically stablish that the rate of convergence is logarithmic.
\end{rem}
\begin{figure}[thpb]
      \centering
      \includegraphics[scale=0.2]{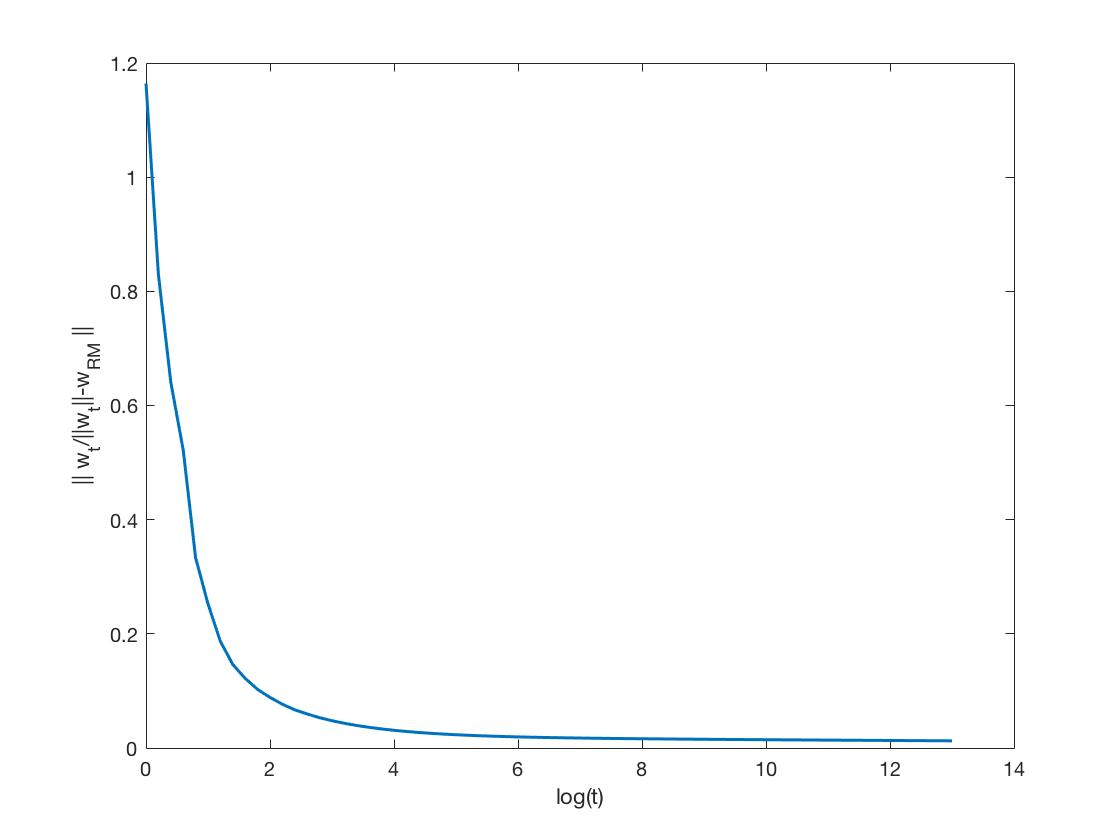}
      \caption{Convergence of GD iterates to the RM classifier. For our experiment we have $n=30$, $p=10$, number of iterations is $10^{13}$, and $\epsilon_i\sim \text{Unif}(0,\frac{1}{\lnorm{\w_M}})$. The distance between the max-margin and the RM classifier is $\lnorm{\frac{\w_M}{\lnorm{\w_M}}-\frac{\w_{RM}}{\lnorm{\w_{RM}}}} = 0.2192$. }
      \label{fig:2}
\end{figure}

\section{Proof of Lemma~\ref{lem:lim_w}}
\label{sec:proof_lemma}
In our proof we use the following lemma which characterizes the behavior of gradient descent iterates on smooth functions.
\begin{lem}[Lemma 10 in \cite{soudry2018implicit}]
\label{lem:bd_grad_sum}
Let $\mathcal L(\w)$ be a $\gamma$-smooth non-negative objective. If $\eta < \frac{2}{\gamma}$, then, for any starting point $\w(0)$, with
the GD sequence
$$\w (t + 1) = \w (t) -  \eta \nabla \mathcal L(\w(t))$$ 
we have that:
$$\sum_{u=0}^{\infty}\lnorm{\nabla \mathcal L(\w(u))}^2<+\infty.$$
\end{lem}
\vspace{-0.5em}
We also use the following corollary.
\begin{cor}
\label{cor1}
For any positive constant $C<{\beta \big(\sigma_{max}(\mathbf X)+\lnorm{\bm \epsilon}\big)^2}$, there exist $R>0$, such that $\lnorm{\nabla^2\mathcal L_{\bm\epsilon}(\w)}<C$ when $\lnorm{\w}>R$.
\end{cor}
The proof is straightforward, by computing the Hessian of $\mathcal L_{\bm \epsilon}(\cdot)$ and using the fact that $\ell(\cdot)$ is twice-differentiable and  $\beta$-smooth.

Since the function $\ell(\cdot)$ is monotonically decreasing we can write,
\begin{equation}
    \mathcal L_{\bm \epsilon}(\w) = \sum_{i=1}^{n}\ell(y_i\x_i^T\w - \epsilon_i\lnorm{\w})
\end{equation}
The gradient of the loss function can be computed as,
\begin{equation}
    \label{eq:der_loss}
    \nabla \mathcal L_{\bm \epsilon}(\w) = \sum_{i=1}^{n}\ell'(y_i\x_i^T\w - \epsilon_i \lnorm{\w})(y_i\x_i - \epsilon_i\frac{\w}{\lnorm{\w}}).
\end{equation}

Consider the sequence $s_t:= \frac{1}{\eta}\w_{RM}^T\w_t$, for $t\in \mathbb N$. First, we show that this sequence is increasing. 
\begin{align}
\label{eq:seq_increase}
&&s_t-s_{t+1} &= \w_{RM}^T\nabla \mathcal L_{\bm \epsilon}(\w_t)\\
&&=\sum_{i=1}^n\ell'&(y_i\x_i^T\w_t - \epsilon_i\lnorm{\w_t})\w_{RM}^T\big(y_i\x_i - \epsilon_i\frac{\w_t}{\lnorm{\w_t}}\big)\nonumber\\
&&\leq \sum_{i=1}^n\ell'&(y_i\x_i^T\w_t - \epsilon_i\lnorm{\w_t})\big(y_i\x_i^T\w_{RM} - \epsilon_i\lnorm{\w_{RM}}\big)\nonumber\\
&&\leq \sum_{i=1}^n\ell'&(y_i\x_i^T\w_t - \epsilon_i\lnorm{\w_t}) <0~,\nonumber
\end{align}
where for the first inequality we used the fact that $\ell'(u)<0$ and Cauchy-Schwartz, and for the second inequality we used the constraints of the optimization~\eqref{eq:robust_max-margin}. \\
Since $\{s_t\}_{t\geq 0}$ is an increasing sequence in $\R$ it either grows to $+\infty$ or approaches a limit value. We analyze each of these cases separately.\\
{\boxed{\text{Case 1}: \lim_{t\rightarrow \infty} s_t = L<+\infty }}\\
When the sequence has a limit, we have $\lim_{t\rightarrow \infty} {s_t-s_{t+1}} = 0$. From the last inequality in~\eqref{eq:seq_increase}, this implies that as $t\rightarrow \infty$,
\begin{equation}
    \ell'(y_i\x_i^T\w_t - \epsilon_i\lnorm{\w_t})\rightarrow 0,~\text{for } 1\leq i\leq n.
\end{equation}
Since $\ell'(u)$ is negative for $u\in\R$, we must have 
\begin{equation}
    \label{eq:case1}
    y_i\x_i^T\w_t - \epsilon_i\lnorm{\w_t}\rightarrow +\infty,~\text{for } 1\leq i\leq n,
\end{equation}
which is (iii). This also implies that $\lnorm{\w_t} \rightarrow \infty$.  Finally, from \eqref{eq:der_loss} we have that $\nabla \mathcal L_{\bm \epsilon}(\w_t)\rightarrow \mathbf 0_p$.
{\boxed{\text{Case 2}: ~\lim_{t\rightarrow \infty} s_t = +\infty}}\\
$\lnorm{\w_t}\geq \frac{\eta s_t}{\lnorm{\w_{RM}}}$ implies that $\lim_{t\rightarrow \infty}{\lnorm{\w_t}}=+\infty$. Using Corollary~\ref{cor1}, for any constant $C<\beta \big(\sigma_{max}(\mathbf X)+\lnorm{\bm\epsilon}\big)^2$, there exists a nonnegative integer $t_0$ such that the second derivative is bounded by $C$ for any $t>t_0$. Hence, we can use the result of Lemma~\ref{lem:bd_grad_sum} with $\eta<2\cdot \beta^{-1}\cdot(\sigma_{\max}(\mathbf X)+\lnorm{\bm \epsilon})^{-2}$ which gives $\lnorm{\nabla \mathcal L_{\bm\epsilon}(\w_t)}\rightarrow 0$ as $t\rightarrow +\infty$. \\
In order to show (iii), we use the last inequality in~\eqref{eq:seq_increase}, as $t\rightarrow \infty$ since $\w_{RM}^T\nabla \mathcal L_{\bm \epsilon}(\w_t)\rightarrow 0$, we have:
\begin{equation}
    \ell'(y_i\x_i^T\w_t - \epsilon_i\lnorm{\w_t})\rightarrow 0,~\text{for } 1\leq i\leq n,
\end{equation}
which gives the desired result.
\section{Proof of Theorem~\ref{thm:main}}
\label{sec:proof_main}
 For the RM classifier, we define the set of support vectors as:
\begin{equation}
    \label{eq:supp_vec}
    \mathcal S=\mathcal S_{RM} := \{i\in [n]:y_i\x_i^T\w_{RM} = 1+\epsilon_i\lnorm{\w_{RM}}\},
\end{equation}
First, we consider the KKT conditions for the optimization~\eqref{eq:robust_max-margin} which gives:
\begin{equation}
    \w_{RM} = \sum_{i\in \mathcal S} \alpha_i \big(y_i\x_i - \epsilon_i \hat \w\big), 
\end{equation}
where $\hat\w := \frac{\w_{RM}}{\lnorm{\w_{RM}}}$ and $\alpha_i\geq 0$. It can be shown that when the data points are drawn from a continuous distribution, for almost every dataset the support vectors are linearly independent and $\alpha_i$'s are all positive~(see also~\cite{ji2018risk} and Appendix B in~\cite{soudry2018implicit}). Given the fact that $-\ell'(\u)$ has a exponential tail, we assume $\alpha,\gamma,\tau, \mu$ are positive constants such that:
\begin{equation}
    \begin{cases}
    -\ell'(u)\leq \gamma \big(1+\exp(-\mu\cdot u)\big)\exp(-\alpha\cdot u),~\text{and},\\
    -\ell'(u)\geq\gamma \big(1-\exp(-\mu\cdot u)\big)\exp(-\alpha\cdot u),
    \end{cases}
\end{equation}
for every $u\geq \tau$.

We define a vector $\tilde \w$ such that:
\begin{equation}
    \exp\big(\tilde \w^T (y_i\x_i - \epsilon_i \hat \w)\big):=\frac{\alpha_i}{\gamma\cdot \eta}~,~~\text{for }i=1,2,\ldots,n.
\end{equation}
Recall that the gradient descent iterates are defined as,
\begin{equation}
    \label{eq:pf_GD_iterates}
    \w_{t+1} - \w_t = -\eta \nabla \mathcal L\big(\w_t\big),~t\in \mathbb N.
\end{equation}
Next, for $t\geq 0$ we define the residual vector $\mathbf r_t\in \R^p$.
\begin{equation}
\label{eq:res_def}
\mathbf r_t: = \w_t - \frac{1}{\alpha}\log(t)\w_{RM} - \tilde \w.
\end{equation}
In our proof, we adopt a similar strategy as~\cite{soudry2018implicit} and bound the norm of the residual vector $\lnorm{\mathbf r(t)}$ by a constant $C$ for every $t\geq 1$. Consider the following equation,
\begin{equation}
    \label{eq:residual_difference}
    \lnorm{\mathbf r_{t+1}}^2 - \lnorm{\mathbf r_{t}}^2 = \lnorm{\mathbf r_{t+1} - \mathbf r_t}^2 + 2~\mathbf r_t^T\big(\mathbf r_{t+1}-\mathbf r_{t}\big). 
\end{equation}
We bound each of the two terms in the RHS of~\eqref{eq:residual_difference}.
We start with bounding the first term in the~\eqref{eq:residual_difference}. We have:
\begin{equation}
    \label{eq:proof_eq_1}
    \begin{aligned}
    &&\lnorm{\mathbf r_{t+1} - \mathbf r_t}^2&= \lnorm{\mathbf w_{t+1} - \mathbf w_t - \w_{RM} \big(\log(\frac{t+1}{t})/\alpha\big)}^2 \\
    &&&\leq \eta^2\lnorm{\nabla \mathcal L(\w_t)}^2 + (\alpha t)^{-2} \lnorm{\w_{RM}}^2 \\
    &&&~~~~+ 2(\eta/\alpha)\log(1+t^{-1}){\w_{RM}}^T \nabla \mathcal L(\w_t)\\
    &&& \leq \eta^2\lnorm{\nabla \mathcal L(\w_t)}^2 + (\alpha t)^{-2} \lnorm{\w_{RM}}^2~.
    \end{aligned}
\end{equation}
Where in the first inequality we replaced $\w_{t+1}-\w_t$ using the gradient descent iterates~\eqref{eq:pf_GD_iterates} along with $\log(1+u)\leq u$, and in the second inequality we exploited the inequality~\eqref{eq:seq_increase} that gives ${\hat \w}^T \nabla \mathcal L(\w(t))<0$. 

Since the norm of $\w_t$ approaches infinity as $t$ grows, when $\eta<2\cdot \beta^{-1}\cdot(\sigma_{\max}(\mathbf X)+\lnorm{\bm \epsilon})^{-2}$ we can use the result of Corollary~\ref{cor1} and Lemma~\ref{lem:bd_grad_sum} to have:
\begin{equation}
    \sum_{t=0}^{\infty}\lnorm{\nabla \mathcal L(\w_t)}<C_1,
\end{equation}
for some constant $C_1>0$. Therefore, we can bound the sum over the first term in~\eqref{eq:residual_difference}.
\begin{equation}
    \label{eq:bd_first}
    \sum_{t\geq1} ||\mathbf r_{t+1}-\mathbf r_t||^2\leq \eta^2 C_1 + \alpha^{-2}\lnorm{\w_{RM}}^2 \sum_{t\geq1} t^{-2} <C_2.
\end{equation}
Next, we will bound the second term in~\eqref{eq:residual_difference}, i.e., $\mathbf r_t^T\big(\mathbf r_{t+1}-\mathbf r_t\big)$. To do so, we first define the constant $\theta$ as follows:
\begin{equation}
    \theta := \min _{i\in \mathcal S^{c}} y_i \x_i\w_{RM} - \epsilon_i\lnorm{\w_{RM}}>1,
\end{equation}
where $\mathcal S^c = [n]-\mathcal S$ indicates the indices of non-support vectors. 
The following lemma provides an upper bound on $\mathbf r_t^T\big(\mathbf r_{t+1}-\mathbf r_t\big)$ for $t\geq1$.
\begin{lem}
\label{lem:bd_inner_prod}
With the assumptions of Theorem~\ref{thm:main}, consider the gradient descent iterates~\eqref{eq:pf_GD_iterates}, $\{\w_t\}_{t\in \mathbb N}$, and the vector $\mathbf r_t$ defined in~\eqref{eq:res_def}. Then, for constants $C\geq 0$ and $t_0\in \mathbb N$, we have:
\begin{equation}
    \mathbf r_t^T\big(\mathbf r_{t+1}-\mathbf r_t\big)\leq Ct^{-\min(\theta, 1+\frac{\mu}{2\alpha})}~,~~\forall t\geq t_0.
\end{equation}
\end{lem}
Using the result of Lemma~\ref{lem:bd_inner_prod}, since $\theta>1$ and $\mu/\alpha>0$ we have:
\begin{equation}
\label{eq:bd_sum_inner}
\begin{aligned}
    &&\sum_{t\geq 0}\mathbf r_t^T\big(\mathbf r_{t+1}-\mathbf r_t\big) &< \sum_{t=1}^{t_0-1}\mathbf r_t^T\big(\mathbf r_{t+1}-\mathbf r_t\big)+ C\sum_{t\geq t_0} t^{-\min(\theta, 1+\frac{\mu}{2\alpha})} \\
    &&&<C_3.
\end{aligned}
\end{equation} 
Therefore, from~\eqref{eq:residual_difference},~\eqref{eq:bd_first}, and~\eqref{eq:bd_sum_inner}, we have,
\begin{equation}
    \lnorm{\mathbf r_k}^2  = \lnorm{\mathbf r_1}^2 + \sum_{t=1}^{k-1} \lnorm{\mathbf r_{t+1}}^2 - \lnorm{\mathbf r_{t}}^2 <C_4~,~~\forall k\geq 1. 
\end{equation}
for a positive constant $C_4$. Consequently, from~\eqref{eq:res_def} we have,
\begin{equation}
    \lnorm{\w_t-\frac{1}{\alpha} \log(t) \w_{RM}}\leq C_4 + \lnorm{\tilde \w},
\end{equation}
By some straightforward calculations we can get,
\begin{equation}
    \lnorm{\frac{\w_t}{\lnorm{\w_t}} - \frac{\w_{RM}}{\lnorm{\w_{RM}}}}^2 \leq 2 \big[\frac{\alpha( C_4+\lnorm{\tilde \w})}{\log(t)\lnorm{\w_{RM}}}\big]^2,
\end{equation}
which gives the desired result, i.e.,
\begin{equation}
    \lim_{t\rightarrow \infty}\lnorm{\frac{\w_t}{\lnorm{\w_t}} - \frac{\w_{RM}}{\lnorm{\w_{RM}}}} = 0.
\end{equation}

\bibliography{library}
\bibliographystyle{plain}

\end{document}

%% file: defs.tex
\usepackage[utf8]{inputenc} 
\usepackage[T1]{fontenc}    
\usepackage{hyperref}       
\usepackage{url}            
\usepackage{booktabs}       
\usepackage{amsfonts}       
\usepackage{amsmath}
\usepackage{nicefrac}       
\usepackage{microtype}      
\usepackage{fullpage}
\usepackage{enumerate}
\usepackage{epstopdf}
\usepackage{mathtools,amssymb,algorithm,caption,cases,bm,float,graphicx,url,color}
\usepackage{subcaption}
\usepackage{sidecap}
\newtheorem{lem}{Lemma}
\newtheorem{theorem}{Theorem}
\newtheorem{defn}{Definition}
\newtheorem{assump}{Assumption}
\newtheorem{rem}{Remark}
\newtheorem{cor}{Corollary}
\newcommand{\w}{\mathbf{w}}
\renewcommand{\v}{\mathbf{v}}

\renewcommand{\u}{\mathbf{u}}

\newcommand{\x}{\mathbf{x}}

\newcommand{\R}{\mathbb{R}}

\newcommand{\lnorm}[1]{\left\Vert {#1} \right\Vert}
